\documentclass{article}

\usepackage{microtype}
\usepackage{graphicx}
\usepackage{subfigure}
\usepackage{booktabs} 

\usepackage{hyperref}

\usepackage[accepted]{icml2023}

\usepackage{amsmath}
\usepackage{amssymb}
\usepackage{mathtools}
\usepackage{amsthm}

\usepackage[capitalize,noabbrev]{cleveref}

\theoremstyle{plain}
\newtheorem{theorem}{Theorem}[section]

\theoremstyle{definition}

\theoremstyle{remark}

\usepackage[textsize=tiny]{todonotes}

\usepackage{bm}
\newcommand{\algnamenospace}{BCART-PCFG}
\newcommand{\algname}{BCART-PCFG }

\icmltitlerunning{Bayesian Decision Trees via Tractable Priors and Probabilistic Context-Free Grammars}

\begin{document}

\twocolumn[
\icmltitle{Bayesian Decision Trees via Tractable Priors and Probabilistic Context-Free Grammars}



\icmlsetsymbol{equal}{*}

\begin{icmlauthorlist}
\icmlauthor{Colin Sullivan}{equal,st}
\icmlauthor{Mo Tiwari}{equal,st}
\icmlauthor{Sebastian Thrun}{st}
\icmlauthor{Chris Piech}{st}
\end{icmlauthorlist}

\icmlaffiliation{st}{Department of Computer Science, Stanford University, Stanford, California, USA}

\icmlcorrespondingauthor{Colin Sullivan}{colins26@stanford.edu}
\icmlcorrespondingauthor{Mo Tiwari}{motiwari@stanford.edu}

\icmlkeywords{Machine Learning, bayesian decision trees, BCART, BART, random forests, ensemble models, ICML}

\vskip 0.3in
]



\printAffiliationsAndNotice{\icmlEqualContribution} 


\begin{abstract}
\label{sec:abstract}
Decision Trees are some of the most popular machine learning models today due to their out-of-the-box performance and interpretability.
Often, Decision Trees models are constructed greedily in a top-down fashion via heuristic search criteria, such as Gini impurity or entropy.
However, trees constructed in this manner are sensitive to minor fluctuations in training data and are prone to overfitting.
In contrast, Bayesian approaches to tree construction formulate the selection process as a posterior inference problem; such approaches are more stable and provide greater theoretical guarantees.
However, generating Bayesian Decision Trees usually requires sampling from complex, multimodal posterior distributions.
Current Markov Chain Monte Carlo-based approaches for sampling Bayesian Decision Trees are prone to mode collapse and long mixing times, which makes them impractical.
In this paper, we propose a new criterion for training Bayesian Decision Trees.
Our criterion gives rise to \algnamenospace, which can efficiently sample decision trees from a posterior distribution across trees given the data and find the maximum a posteriori (MAP) tree.
Learning the posterior and training the sampler can be done in time that is polynomial in the dataset size.
Once the posterior has been learned, trees can be sampled efficiently (linearly in the number of nodes).
At the core of our method is a reduction of sampling the posterior to sampling a derivation from a probabilistic context-free grammar.
We find that trees sampled via \algname perform comparable to or better than greedily-constructed Decision Trees in classification accuracy on several datasets.
Additionally, the trees sampled via \algname are significantly smaller -- sometimes by as much as 20x.
\end{abstract}


\section{Introduction}
\label{sec:introduction}

Decision Trees (DTs) and related ensemble methods such as Random Forest (RF) are widely-used machine learning methods. 
DTs recursively partition the feature space by performing if/then/else comparisons on the feature values of a given datapoint. 
Each internal node in a DT consists of a subset of the data that is split into two child nodes based on that node's decision rule, which partitions points according to whether a given feature value is greater than or less than a fixed threshold.
At the end of this process, a prediction label is generated for the datapoint based on the leaf node of the decision tree in which it falls.

In most popular implementations of these models, each base learner (a single DT) is trained greedily from the top-down.
More specifically, each node is split according to the single best rule that would lower its heterogeneity.
Which decision rule is best is determined by heuristic proxies for end performance, such as Gini impurity or entropy in classification and mean squared error (MSE) in regression.

Single DTs trained in this greedy fashion, however, can often be sensitive to minor fluctuations in the training data and have limited theoretical guarantees.
RF was developed to combat overfitting of single DTs \cite{tinkamhoRandomDecisionForests1995,breimanRandomForests2001b}.
In RF, each tree's decision rules are determined by looking at only a random subset of available features and using a bootstrap sample of the training data.
Popular extensions of the RF method include boosted methods such as XGBoost and LightGBM \cite{chenXGBoostScalableTree2016,keLightGBMHighlyEfficient2017}.


\begin{figure*}
\begin{center}
\centerline{\includegraphics[width=\columnwidth]{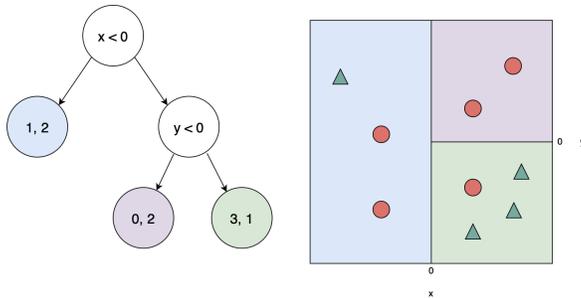}}
\caption{An example decision tree based on the features $x$ and $y$. Points with $x \geq 0$ and $y > 0$ can be classified as one class (green triangles), whereas points in the other leaf nodes are classified as another (red circles).}
\label{fig:decision-tree}
\end{center}
\end{figure*}


Such DT ensemble methods are widely used in practice and improve upon individual DTs' performance and stability.
However, these methods lose their simplicity and interpretability.
For these reasons, Bayesian CART (BCART) Model Search was proposed as an alternative to greedy DT construction algorithms for constructing individual DTs.
BCART is a probabilistic method that assigns a prior probability distribution over the set of potential trees and updates the corresponding posterior distribution given the data.
Individual trees generated by BCART, which we call Bayesian Decision Trees (BDTs), have been shown to have performance commensurate with that of entire greedily-constructed RFs \cite{lineroReviewTreebasedBayesian2017}.
Furthermore, the posterior generated by the BCART method can be used to quantify uncertainty in the tree selection process.

In the BCART formulation, each individual DT is sampled from the posterior distribution of trees given the data, typically using Markov chain Monte Carlo (MCMC) methods.
These methods typically only permit local transitions between trees, such as growing or pruning trees by a single node \cite{chipmanBayesianCARTModel1998c, denisonBayesianCARTAlgorithm1998b}.
MCMC-based BCART methods have been shown to result in smaller and equally performant individual trees when compared with greedily constructed DTs and RF \cite{lineroReviewTreebasedBayesian2017}.
Unfortunately, if the posterior over trees is naturally complex and multimodal, MCMC-based methods encounter long mixing times and exhibit mode collapse \cite{ronenMixingTimeLower2022a,yannottyModelMixingUsing2023a}.
In practice, this often renders MCMC-based methods impractical.


\textbf{Contributions}: In this work, we propose a new Bayesian prior as a criterion for the Bayesian sampling objective when sampling decision trees. 
This criterion leads to an analytically tractable posterior distribution over trees given the data.
We propose \algname to sample from this posterior;
\algname is able to sample trees efficiently. 
At the heart of our method is a theoretical connection between our choice of prior and sampling from a probabilistic context-free grammar.
Sampled trees demonstrate performance comparable to or better than existing, greedily-constructed trees, and are often significantly smaller.
Furthermore, the sampling procedure is easily parallelizable and \algname has only a single hyperparameter, which makes hyperparameter tuning simple.


\section{Related Work}
\label{sec:related_work}

DTs are a popular for their strong out-of-the-box performance; DTs still outperform neural network approaches on tabular data \cite{grinsztajnWhyTreeBasedModels2022a}.
Popular algorithms for greedily training DTs include CART \cite{breimanClassificationRegressionTrees1984a}, ID3 \cite{quinlanInductionDecisionTrees1986b}, and C4.5 \cite{quinlanImprovedUseContinuous1996a}. 
Recently, other work has focused on using metaheuristics for the training criterion \cite{rivera-lopezInductionDecisionTrees2022b} or accelerating the training procedure via adaptive sampling \cite{tiwariMABSplitFasterForest2022a}.

BDTs were first introduced simultaneously in \citet{denisonBayesianCARTAlgorithm1998b} and \citet{chipmanBayesianCARTModel1998c}.
Both works propose a prior over tree structures and suggest sampling from the posterior via MCMC methods.
Since these papers, most followup work on BDTs has focused on accelerating the sampling procedure, e.g., via particle filtering \cite{lakshminarayananTopdownParticleFiltering2013} or other methods \cite{geelsTaxicabSamplerMCMC2022a}, and extending capabilities to larger datasets \cite{denisonBayesianPartitionModelling2002a}.

Significant recent work demonstrates the advantages of Bayesian methods for generating tree ensembles as opposed to the greedy, algorithmic method; see \citet{lineroReviewTreebasedBayesian2017} for a recent review. In this line of work, \citet{matthewBayesianEmpiricalBayesian2015} shows promising results for Bayesian tree ensembles but stops short of a fully Bayesian approach; instead, they sample decision stumps greedily and grow these stumps with a Bayesian approach. \citet{nutiExplainableBayesianDecision2021b} showed that what they call greedy modal tree (GMT) can match the performance of the best trees constructed by greedy splitting criteria. 
However, their approach only considers a posterior across single-level splits and greedily constructs trees from the top down using this metric.

One of the most successful applications of BDTs is in Bayesian Additive Regression Trees (BARTs) \cite{chipmanBARTBayesianAdditive2010a}, which generate multiple Bayesian decision trees in order to partition data for downstream tasks, and its extensions \cite{deshpandeNewBARTPrior2022a, maiaGPBARTNovelBayesian2022a, nutiAdaptiveBayesianReticulum2019}.
However, similar to BDTs, these approaches often suffer from mode collapse and large mixing times, often rendering them unusable \cite{ronenMixingTimeLower2022a, yannottyModelMixingUsing2023a}.

We note that our approach is inspired by connections to probabilistic context-free grammars \cite{johnsonBayesianInferencePCFGs2007b, kimCompoundProbabilisticContextFree2019a, lieckRecursiveBayesianNetworks2021a, pynadathGeneralizedQueriesProbabilistic1998a, toutanovaFeatureSelectionRich2002a, saranyadeviContextfreeGrammarBased2021a, stolckeInducingProbabilisticGrammars1994b, teichmannExpressingContextFreeTree2016a, vazquezGramMLExploringContextFree2022a, hwangInducingProbabilisticPrograms2011a}.
Our sampling approach is also inspired by recently developed methods in Generative Flow Networks \cite{bengioGFlowNetFoundations2022a,doImprovingGenerativeFlow2022a,panGenerativeAugmentedFlow2022a,zimmermannVariationalPerspectiveGenerative2022a}.


\section{Preliminaries}
\label{sec:preliminaries}

We consider a dataset $X, Y$ of $n$ independent observations $(\textbf{x}^{(1)}, y^{(1)}), \dots, (\textbf{x}^{(n)}, y^{(n)})$.
we focus on the classification setting for simplicity, though our results also apply to the regression setting.
In classification, each vector $\textbf{x}^{(i)}$ consists of $d$ features $\textbf{x} = (x_1, \dots, x_d)$, and each label $y \in \{1, \dots, C\}$. 

Throughout this work, we let $T$ denote a binary tree and let $b(T)$ denote the number of terminal (leaf) nodes of $T$.
We let $M_j \in b(T)$ denote the terminal nodes of tree $T$ for $j = 1, 2, \ldots, b(T)$ and $N_k$ denote the non-terminal (internal) nodes of $T$ for $k = 1, 2, \ldots, b(T) - 1$.

Each internal node $N_k$ of $T$ is associated with a splitting rule, denoted $s_k$, which can be written as the subset $\{x_i < v\}$ for some $v \in \mathbf{R}$. $\overline{s}$ denotes the negation of the splitting rule $s$ (e.g., if $s$ is $x_1 < 3$ then $\overline{s}$ is $x_1 \geq 3$).
Points satisfying the splitting rule of an internal node are sent to its left child node and all others (those that satisfy the complement) are sent to its right child node. This splitting process continues until these points reach a terminal node in $T$.

Each node of $T$ is also associated with a bounding box that is axis-aligned (in feature space) and determined by the splits performed in $T$ to reach $N$ from the root node, which itself is denoted by $N_0$. 
With a slight abuse of notation, we use $N$ to refer to a node or its associated bounding box, with the meaning clear from context. 
Similarly, we use $N \cap s$ and $N \cap \overline{s}$ to denote the children of $N$ created by splitting according to rule $s$, i.e., the left and right children of $N$ when it is split by rule $s$, respectively. 

If, for example, a node $N_4$ has three ancestors $N_0, N_2, N_3$ in which it is in the left, right, and left subtrees of, respectively, then the bounding box associated with $N_4$ is $N_0 \cap s_0 \cap \overline{s_2} \cap s_3$.

We let $X_N$ and $Y_N$ or $X_M$ and $Y_M$ denote the datapoints and their associated labels, respectively, at an internal node $N$ or a terminal node $M$. Two bounding boxes are considered equivalent if they contain the same datapoints.

\subsection{Probabilistic Context Free Grammars}

We briefly recapitulate Probablistic Context-Free Grammars (PCFGs). 
PCFGs are formally defined as a tuple of five elements: the start symbol, a set of nonterminal symbols, a set of terminal symbols, a set of production rules, and the set of probabilities over those production rules.
A derivation in a PCFG consists of recursively applying production rules to the start symbol.
The probability of an ordered derivation occurring is equal to the product of the probabilities for the production rules applied to generate that derivation.


\section{Bayesian Decision Trees and \algname}
\label{sec:bdts}

A Bayesian Decision Tree (BDT) is a binary tree $T$ and parameter $\Theta = (\bm\theta_1, \dots, \bm\theta_{b(T)})$ that parameterizes the distributions over labels $f(y | \bm\theta_1, \dots, f(y | \bm\theta_{b(T)})$ in each of the terminal nodes. 
Each distribution $f(y | \bm\theta_i)$ is a categorical distribution parameterized by $\theta_{i, 1}, \dots, \theta_{i, C}$ where each $\theta_{i, c}$ denotes the probability that a datapoint in terminal node $i$ is assigned label $c$. 
With this notation, we state the following theorem:

\begin{theorem}
\label{thm:y_given_xt}
Assume that labels are drawn independently from their respective distributions conditioned on $T$ and $\Theta$.  
Further assume a Dirichlet prior across these distributions of labels, $\bm\alpha = (\alpha_1, \dots, \alpha_C)$.
Then the likelihood for dataset labels $Y$ given the predictors $X$ and binary tree $T$ is

\begin{equation}
\label{eqn:likelihood}
    \Pr[Y \mid X, T] = \prod_{j=1}^{b(T)} \frac{\beta(\mathbf{n_j} + \bm\alpha)}{\beta(\bm\alpha)}
\end{equation}

where $\mathbf{n_j} = (n_{j, 1}, \dots, n_{j, C})$ is a vector where the $c$th coordinate is the number of occurrences of label $c$ in the $j$th terminal node, $\beta$ is the $C$-beta function given by

\begin{equation}
\label{eqn:gamma}
    \beta(\gamma_1, \dots, \gamma_C) = \frac{\prod_{i=1}^C \Gamma(\gamma_i)}{\Gamma\left(\sum_{i=1}^C \gamma_i\right)},
\end{equation}

and $\Gamma$ is the gamma function given by $\Gamma(n) = (n - 1)!$ for integer $n$.
\end{theorem}

\begin{proof}

By the independence assumption, we may write
\begin{equation}
    \Pr[Y | X, T] = \prod_{i = 1}^n \Pr[y_i | x_i, T].
\end{equation}

Furthermore, by grouping the terms in the product on the right hand side by leaf node, we may write

\begin{align}
    \prod_{i = 1}^n \Pr[y_i | x_i, T] &= \prod_{j = 1}^{b(T)} \prod_{r=1}^{n_j} \int_{\bm\theta} \Pr[y_i | \bm\theta_j, x_i, T] \Pr[\bm\theta_j] d\bm\theta_j \\
    &= \prod_{j = 1}^{b(T)} \prod_{r=1}^{n_j} \int_{\bm\theta} \Pr[y_i | \bm\theta_j] \Pr[\bm\theta_j] d\bm\theta_j
\end{align}

where $n_j$ is the number of datapoints in leaf $j$ of tree $T$ and since $Pr[y_i | \bm\theta_j, x_i, T] = Pr[y_i | \bm\theta_j]$ as the prediction $y_i$ only depends on the parameter $\bm\theta_j$. 
Furthermore, 

\begin{align}
    \prod_{r=1}^{n_j} \int_{\bm\theta} \Pr[y_i | \bm\theta_j] \Pr[\bm\theta_j]d\bm\theta_j &= \int_{\bm\theta} \prod_{r=1}^{n_j}  \Pr[y_i | \bm\theta_j] \Pr[\bm\theta_j]d\bm\theta_j \\
    &= \frac{\beta(\mathbf{n_j} + \bm\alpha)}{\beta(\bm\alpha)}
\end{align}

for each $j$, where we used the fact that $ \Pr[\bm\theta_j] \propto \theta_{j, 1}^{\alpha_1 - 1} \ldots \theta_{j,C}^{\alpha_C - 1}$ and results from \cite{denisonBayesianCARTAlgorithm1998b}. Thus 

\begin{equation}
    \Pr[Y | X, T] = \prod_{j=1}^{b(T)}  \frac{\beta(\mathbf{n_j} + \bm\alpha)}{\beta(\bm\alpha)}
\end{equation}

as desired.
\end{proof}

Theorem \ref{thm:y_given_xt} allows us to compute the the likelihood of labels $Y$ occurring for data $X$ under a given tree structure $T$ and choice of leaf node probability distributions $f(y | \bm\theta_1), \dots, f(y | \bm\theta_{b(T)})$ where each $f$ is a Dirichlet distribution parameterized by its corresponding $\bm\theta$.
For example, with the tree $T$ and points in Figure \ref{fig:decision-tree}, assuming a uniform prior where $\bm\alpha = (1, 1)$, the conditional likelihood $\Pr[Y | X, T]$ of generating the provided labels is $\frac{\beta(2, 3)\beta(1, 3)\beta(4, 2)}{\beta(1, 1)^3}$.

Our goal, however, is to sample trees from the posterior $\Pr [T \mid X, Y]$. To do so, we must also choose a prior across trees $\Pr[T | X]$.
As noted in previous work \cite{buntineLearningClassificationTrees1992}, the support of this distribution will depend on our data $X$ due to the following natural constraints, which we call the Tree Constraints:
\begin{enumerate}
\label{lst:constraints}
    \item the bounding boxes of nodes in $T$ must be nonempty, and
    \item any splits on a node which result in the same two divided bounding boxes are considered equivalent.
\end{enumerate}
These constraints significantly reduce the considered space of trees considered.
In particular, the number of possible trees is finite.

We use a basic prior mentioned in previous work \cite{denisonBayesianCARTAlgorithm1998b,chipmanBayesianCARTModel1998c} which scales exponentially with the number of terminal nodes in the tree. 
This choice of prior biases the tree search towards smaller trees:
\begin{equation}
\label{eqn:prior}
    \Pr[T \mid X] \propto \phi^{-b(T)}
\end{equation}

\begin{theorem}
\label{thm:posterior}
The posterior over trees resulting from the likelihood in Equation \eqref{eqn:likelihood} and prior in Equation \eqref{eqn:prior} is:
\begin{equation}
    \Pr[T \mid X, Y] \propto \left( \prod_{k=1}^{b(T)} \frac{\beta(\mathbf{n_k} + \bm\alpha)}{\beta(\bm\alpha)} \right) \left( \phi^{-b(T)} \right)
\end{equation}
\end{theorem}

\begin{proof}
A straightforward application of Bayes' rule yields

\begin{align}
    \Pr[T \mid X, Y] & \propto \Pr[Y \mid X, T]\Pr[T \mid X] \\
    &= \left( \prod_{k=1}^{b(T)} \frac{\beta(\mathbf{n_k} + \bm\alpha)}{\beta(\bm\alpha)} \right) \left( \phi^{-b(T)} \right)
\end{align}
    
\end{proof}

Theorem \ref{thm:posterior} provides us with a posterior for the trees given the data.
In later sections, we will define a method of sampling trees according to this posterior.

\section{Methodology}
\label{sec:methodology}

\subsection{Bounding Box Score}
Let $S(N)$ denote the set of splits of the form $s = \{x_i < d\}$ such that $N \cap s$ and $N \cap \overline{s}$ result in two nonempty bounding boxes. 
As before, we let any two such splits that result in the same two bounding boxes be considered equivalent.
$S(N)$ contains one representative from each equivalence class. 

We define the score of a bounding box to be:
\begin{equation}
    Q(N) \coloneqq L(N) + \frac{1}{\phi} \sum_{s \in S(N)} Q(N \cap s) Q(N \cap \overline{s})
\end{equation}
where $L(N) \coloneqq \Pr[Y_N | X_N, \odot]$ is the likelihood of assigning the labels of $Y_N$ given an empty tree $\odot$. From Theorem \ref{thm:y_given_xt}, we have
\begin{equation}
    L(N) = \frac{\beta(\bm{n}(N) + \bm\alpha)}{\beta(\bm\alpha)}
\end{equation}
where $\bm{n}(N) \in \mathbf{R}^C$ represents the number of datapoints of each category in bounding box $N$.

\subsection{Probabilistic Context-Free Grammar of Bounding Boxes}
Using the score function $Q$, we define a specific PCFG across bounding boxes called PCFG*.
In PCFG*, the start symbol is outermost bounding box $N_0$ which contains all indices in our dataset.
The non-terminal symbols of PCFG* are the bounding boxes associated with non-terminal nodes.
The terminal symbols are the bounding boxes associated with terminal nodes.
We define two productions rules for every non-terminal symbol, specified as follows:
\begin{itemize}
    \item split rule: $N \rightarrow N \cap s, N \cap \bar{s}$
    \item stop rule: $N \rightarrow M$
\end{itemize}
Each split rule represents a division of one bounding box into two smaller non-empty bounding boxes. A stop rule represents a transition from a non-terminal bounding box to a terminal bounding box which contains the same points but will not split further.
We define the probability of split production rule $N \rightarrow N \cap s, N \cap \overline{s}$ to be:
\begin{equation}
    \Pr[N \rightarrow N \cap s, N \cap \overline{s}] = \frac{Q(N \cap s) Q(N \cap \overline{s})}{\phi Q(N)}
\end{equation}
and the probability of a stop production rule $N \rightarrow M$ t as:
\begin{equation}
    \Pr[N \rightarrow M] = \frac{L(N)}{Q(N)}
\end{equation}


Note that a string derived by PCFG* represents an ordered set of bounding boxes; the parse tree of this string corresponds directly with a decision tree containing the same internal/terminal nodes and internal node splits. 
Furthermore, every decision tree satisfying the Tree Constraints can be represented by some such parse tree, and so there is a bijection between parse trees in PCFG* and decision trees, where each of the production rules in the tree map directly to a split rule of the corresponding internal node of the decision tree.
We denote by $P_T$ the parse tree associated with decision tree $T$ under this bijection.
The probability that PCFG* generates a particular parse tree $P_T$ or decision tree $T$ is the product of its production rule probabilities. 
We are now ready to state our main result:

\begin{theorem}
\label{thm:parse_tree_prob}
The probability of a parse tree $P_T$ to be generated in PCFG* is
\begin{align}
    \Pr[P_T] &= \Pr[T | X, Y] \\
    &\propto \left[ \prod_{k=1}^{b(T)} \frac{\beta(\mathbf{n_k} + \bm\alpha)}{\beta(\bm\alpha)}\right] \left[ \phi^{-b(T)} \right]
\end{align}
\end{theorem}

\begin{proof}
Suppose tree $T$ is generated by a sequence of splits $s_1, s_2, \ldots, s_t$ where each $s_i$ operates on node $N_i$. 
Note that all the $N_i$'s are distinct nodes as each node can only be split once.
Without loss of generality, assume splits $s_1, s_2, \ldots, s_b \neq \odot$ are the nonterminating splits and $s_{b+1} = \ldots =  s_t = \odot$ are the terminating splits.
The likelihood of $P_T$ being generated by the associated production rules in PCFG* are

\begin{align}
    &Pr[P_T] \\
    &=\left( \prod_{q = 1}^b \Pr[Q \rightarrow (Q \cap s_q)(Q \cap \bar{s}_q)] \right) \\ 
    &\hspace{50pt} \times \left( \prod_{p = b+1}^t \Pr[Q \rightarrow (Q \cap s_p)] \right) \\
    &=\left( \prod_{q = 1}^b \frac{Q(N_q \cap s_q)Q(N_q \cap \bar{s}_q)}{\phi Q(N_q)} \right) \\
    &\hspace{50pt} \times \left( \prod_{p = b+1}^t \frac{L(N_p)}{Q(N_p)} \right) 
\end{align}

Note that every value $Q(N_p)$ that appears in the denominator of the second product must appear exactly once in the numerator of the first product.
Conversely, every term of the form $Q(N_q \cap s_q)$ or $Q(N_q \cap \bar{s}_q)$ in the numerator of the first product must appear exactly once in the denominator of the second product.

As such, the product telescopes and each $Q$ term cancels except for $Q(N_0)$. We are left with
\begin{align}
    &\Pr[P_T] = \\
    & \frac{\phi^{1-b(T)}}{Q(N_0)} \left( \prod_{p = b+1}^t L(N_p) \right)  \\
    & \propto \phi \Pr[T | X] \left( \prod_{p = b+1}^t L(N_p) \right)  \\
    & \propto \Pr[T | X] \left( \prod_{p = b+1}^t L(N_p) \right)  \\
    &\propto \Pr[T | X] \prod_{\text{leaf nodes $M'$ in $T$}} \Pr[Y_{M'} | X_{M'}, \odot] \\
    &=  \Pr[Y | X, T] \Pr[T | X] \\
    &\propto \Pr[T | Y, X]
\end{align}

where in the antepenultimate line we used that each $M_p$ is a terminal node and in the penultimate line we used independence of the labels in each leaf node conditioned on their corresponding datapoints. 

Since $\Pr[P_T] \propto Pr[T|Y, X]$ and there is a bijection between the supports of each distribution, we must have $\Pr[P_T] = Pr[T|Y, X]$ under this same bijection.
\end{proof}

Theorem \ref{thm:parse_tree_prob} states that to sample BDTs from our Bayesian objective $\Pr[T | X, Y]$, we can sample parse trees in PCFG*.
The problem of sampling BDTs from $\Pr[T | X, Y]$ is thus reduced to sampling parse trees.

\subsection{Computing Score}
In order to sample from PCFG*, we must compute score function $Q$. 
Bounding boxes have a natural scoring given by the number of points they contain, which allows us compute their $Q$ values using dynamic programming naively in total time $O(n \cdot n^{2d})$ where $n$ is the total number of points in our dataset, $d$ is the number of features, and $O(n^{2d})$ is a loose upper bound on the number of bounding boxes.
As a result, we see that Bayesian decision trees can be sampled directly from this posterior in linear time after polynomial time (in $n$, with $d$ fixed) precomputation of $Q$.

\subsection{MAP Tree}
The maximum a posteriori (MAP) decision tree can be also be discovered in a similar fashion. We define a modified score function $Q_{\max}$ as follows
\begin{multline*}
    Q_{\max}(N) = \max ( L(N), \\
    \frac{1}{\phi} \max_{s \in S(N)} Q_{\max}(N \cap s) Q_{\max}(N \cap \overline{s}).
    )
\end{multline*}

With a new PCFG defined using this score function, we can choose the rule with the highest probability recursively at each step.
Algorithm \ref{alg:path_sample_best} describes how to sample from the PCFG once $Q_{\max}$ has been computed.

\subsection{Implicitly Generated Trees}
Since predictions are made only on the terminal node in which a query falls within a decision tree, generating a full tree is not necessary for \algname during inference.
Instead, only the path to the terminal nodes specific to the provided query needs to be sampled.

\begin{algorithm}[tb]
\caption{Path Sample for Query Response}
\label{alg:path_sample}
\textbf{Input}: Query point $q$ \\
\textbf{Output}: Prediction $y$
\begin{algorithmic}[1] 
\STATE $N \leftarrow N_0$
\STATE $terminal \leftarrow$ False
\WHILE{$terminal = $ False}
    \STATE Sample production rule $r$ according to the production rules probabilities in PCFG* for bounding box $N$
    \IF{$r = N \rightarrow M$}
        \STATE $N \leftarrow M$
        \STATE $terminal \leftarrow$ True
    \ELSE
    \STATE $s \leftarrow $ the production rule associated with split rule $r$
    \IF{$q \in N \cap s$} 
        \STATE $N \leftarrow N \cap s$
        \ELSE
            \STATE $N \leftarrow N \cap \overline{s}$
        \ENDIF
    \ENDIF
\ENDWHILE
\STATE Draw prediction $y$ from $Pr[y | M]$
\STATE \textbf{return} $y$
\end{algorithmic}
\end{algorithm}


In this fashion, multiple paths can be sampled quickly to simulate the prediction of an ensemble of Bayesian Decision Trees sampled from the posterior.
The exact query response of an ensemble across all supported trees can also be computed using dynamic programming in polynomial time.
A similar method can also be used to determine the MAP tree's response to future queries using the precomputed $Q_{\max}$ function.

\begin{algorithm}[tb]
\caption{Path Sample for Best Tree Query Response}
\label{alg:path_sample_best}
\textbf{Input}: Query point $q$ \\
\textbf{Output}: Prediction $y$
\begin{algorithmic}[1] 
\STATE $N \leftarrow N_0$
\STATE $terminal \leftarrow$ False
\WHILE{$terminal = $ False}
    \STATE $s \leftarrow \frac{1}{\phi}$argmax$_{s' \in S(N)} Q_{\max}(N \cap s')Q_{\max}(N \cap \bar{s'})$
    \IF{$L(N) > \frac{1}{\phi}Q_{\max}(N \cap s)Q_{\max}(N \cap \bar{s})$}
        \STATE $N \leftarrow M$
        \STATE $terminal \leftarrow \text{True}$
    \ELSE
        \IF{$q \in N \cap s$}
            \STATE $N \leftarrow N \cap s$
        \ELSE
            \STATE $N \leftarrow N \cap \overline{s}$
        \ENDIF
    \ENDIF
\ENDWHILE
\STATE Draw prediction $y$ from $Pr[y | M]$
\STATE \textbf{return} $y$
\end{algorithmic}
\end{algorithm}

\section{Experiments}
\label{sec:experiments}

\subsection{Datasets:}
\label{subsec:datasets}

We test the performance of \algname and other baseline algorithms across one synthetic dataset and four real-world datasets from the UCI Machine Learning Repository \cite{Dua:2019}. All datasets are approximately balanced.

\textbf{XOR dataset:} In the synthetic XOR dataset, each datapoint contains 20 features, each of which is a binary feature that is $1$ with probability $0.5$. The label for each datapoint is the 4-way XOR of the first four features. In this dataset, only the first four features are relevant to the label and the remaining 16 are not predictive of the label. A specific binary decision tree with 31 total nodes (16 leaf nodes, one for each possible value of the first four bits) should perfectly classify this dataset. 

\textbf{Iris dataset:} The Iris dataset consists of 50 different samples of each of three different species of Iris flowers. Each datapoint has four dimensions: the length and width of the sepals and petals. The dataset presents a classification task with three possible labels, one corresponding to each species.

\textbf{Caesarean Section dataset:} The Caesarian Section dataset consists of data for 80 different pregnancies. Each pregnancy corresponds to a datapoint with five features: maternal age, delivery number, delivery time (early, timely, or late), blood pressure (low, normal, or high), and presence of heart problems. The dataset presents a classification task with two possible labels corresponding to whether a Caesarian section was performed for the delivery. 

\textbf{Haberman dataset:} The Haberman dataset consists of data for 306 different patients who underwent surgery for breast cancer. Each patient corresponds to a datapoint with three features: patient age, year in which the surgery took place, and number of positive axillary nodes detected. The dataset presents a binary classification task corresponding to whether each patient survived for over 5 years post-surgery.

\textbf{Vertebral Column dataset:} The Vertebral Column dataset consists of data for 310 different orthopaedic patients. Each patient corresponds to a datapoint with six features corresponding to biomechanical attributes derived from the shape and orientation of the pelvis and lumbar spine: pelvic incidence, pelvic tilt, lumbar lordosis angel, sacral slope, pelvic radius, and grade of spondylolisthesis. The dataset presents a binary classification task corresponding to whether each patient displayed orthopaedic abnormalities.

\textbf{Preprocessing:} On all datasets, any features consisting of more than 10 distinct values were bucketed to 10 values. For the Vertebral Column dataset, feature values were bucketed to 5 values. No hyperparameter tuning was necessary for BDTs as ln$\phi$ was set to $2$, so no validation set was necessary. Hyperparameters for greedily-constructed trees were chosen to be their default values.

\textbf{Testing:} 10-fold cross validation was run and average accuracy was computed for each approach on all datasets.

\textbf{Baselines:} We compare the best BDT found by \algname to the best DT found by greedy splitting (we call this the ``single-best'' setting). Additionally, we compare an ensemble of BDTs to RF (we call this the ``ensemble'' setting).


\textbf{Metrics:} We evaluate our algorithm against baselines by comparing two metrics: classification accuracy and tree size (number of total nodes). Higher accuracies and/or smaller tree sizes are better.

\subsection{Results:}

Across all datasets, \algname performs either comparably to or better than RF in both the single-best and ensemble settings in accuracy. Additionally, \algname produces consistently smaller models to achieve the same performance in both settings.

On the XOR dataset, we note \algname is able to find trees that are of minimal size and achieve perfect classification accuracy. RF and the best single decision tree, however, are unable to perform as well, even though they construct significantly larger models. Inspection of greedily constructed DTs reveals that they often split on uninformative features, which increase model size without improving accuracy.

\begin{table}[ht]
\centering
\begin{tabular}{l|ll}
\textbf{Dataset}    & \textbf{Best CART} & \textbf{Best \algname} \\ \hline
\textbf{Iris}       & 0.967 ± 0.069               & 0.967 ± 0.092                   \\
\textbf{Caesarian}  & 0.525 ± 0.193               & \textbf{0.562 ± 0.311}          \\
\textbf{Haberman}   & 0.706 ± 0.219               & \textbf{0.719 ± 0.166}          \\
\textbf{Vertebral}  & 0.713 ± 0.167               & \textbf{0.716 ± 0.215}          \\
\textbf{Hidden XOR} & 0.718 ± 0.217               & \textbf{1.0 ± 0.0}             
\end{tabular}
\caption{Comparison of classification accuracies (\%) for the best-performing CART tree and the best-performing \algname tree on each dataset. Higher is better. On four of five datasets, \algname outperforms CART. Means and 95\% confidence intervals are obtained over 5 random trials.}
\end{table}

\begin{table}[ht]
\begin{tabular}{l|ll}
\textbf{Dataset}    & \textbf{RF} & \textbf{\algname Ensem.} \\ \hline
\textbf{Iris}       & 0.96 ± 0.091         & \textbf{0.967 ± 0.092}              \\
\textbf{Caesarian}  & 0.55 ± 0.263         & \textbf{0.562 ± 0.289}              \\
\textbf{Haberman}   & 0.703 ± 0.197        & \textbf{0.716 ± 0.205}              \\
\textbf{Vertebral}  & 0.726 ± 0.191        & \textbf{0.732 ± 0.211}              \\
\textbf{Hidden XOR} & 0.78 ± 0.121         & \textbf{1.0 ± 0.0}                 
\end{tabular}
\caption{Comparison of classification accuracies (\%) for RF and an \algname ensemble on each dataset. Higher is better. On all five datasets, \algname outperforms CART. Means and 95\% confidence intervals are obtained over 5 random trials.}
\end{table}

\begin{table}[ht]
\begin{tabular}{l|ll}
\textbf{Dataset}    & \textbf{Best CART} & \textbf{Best \algname} \\ \hline
\textbf{Iris}       & 19.6 ± 4.153       & \textbf{7.0 ± 0.0}               \\
\textbf{Caesarian}  & 70.2 ± 2.024       & \textbf{3.2 ± 2.225}             \\
\textbf{Haberman}   & 161.0 ± 17.335     & \textbf{5.6 ± 3.227}             \\
\textbf{Vertebral}  & 106.2 ± 8.097      & \textbf{15.6 ± 4.545}            \\
\textbf{Hidden XOR} & 392.2 ± 103.503    & \textbf{31.0 ± 0.0}             
\end{tabular}
\caption{Comparison of model sizes (number of nodes) for the best-performing CART tree and the best-performing \algname tree on each dataset. Smaller is better. On all five datasets, \algname outperforms CART. Means and 95\% confidence intervals are obtained over 5 random trials.}
\end{table}



\section{Conclusions and Limitations}
\label{sec:conclusions_and_limitations}

Our results demonstrate that \algname achieves similar performance to more common, greedily-constructed trees and constructs significantly smaller models.
This suggests an application to compute-constrained hardware, such as embedded systems.
We also note that \algname is expected to excel on large datasets in which features interact to produce a label, such as the XOR dataset, because \algname is able to do a global optimization over several splits rather than greedily splitting as in CART.
A constraint of our work is that \algname assumes a specific prior distribution over trees which allows us a subtree-decomposable objective function; while our experiments demonstrate this of prior works well in practice, other priors may be desirable in other settings. 
Future work will focus on discovering improved priors which fit into this framework.
A further, significant limitation of our work is that our \algname is restricted to datasets with small dimension due to the exponential scaling in $d$.
In practice, it may be possible to perform some form of dimensionality reduction on the dataset to make \algname feasible on high-dimensional datasets or to approximate the score function using machine learning methods. Since \algname outperforms greedy decision-tree-generation methods when labels depend on more complex interaction between features, we anticipate such an extension of our work to higher-dimensional datasets would be valuable.
Another significant constraint is that \algname requires significant additional storage space compared to the greedy approach, and significantly larger training time for training only a small number (e.g., $1$) of trees.
In practice, it may therefore be best to use \algname in settings with significant computational resources (e.g., server-side) before providing the best single tree to more compute-constrained devices (e.g., smartphones).

\bibliography{updated}
\bibliographystyle{icml2023}

\include{8-Appendix1}

\end{document}